\documentclass{article}

\usepackage{arxiv}

\usepackage[utf8]{inputenc}
\usepackage[T1]{fontenc}
\usepackage{hyperref}
\usepackage{url}
\usepackage{booktabs}
\usepackage{amsfonts}
\usepackage{nicefrac}
\usepackage{microtype}
\usepackage{multirow}
\usepackage{xcolor}
\usepackage{amsmath,amssymb,amsthm}
\usepackage{siunitx}
\usepackage{algorithmic}
\usepackage{algorithm}
\usepackage{tabularx}
\usepackage{graphicx}
\usepackage{caption}
\usepackage{natbib}
\usepackage[table]{xcolor}
\usepackage{subcaption}
\usepackage{enumitem}

\newtheorem{theorem}{Theorem}
\newtheorem{proposition}[theorem]{Proposition}

\newcommand{\GW}{\operatorname{GW}}
\newcommand{\LGW}{\mathcal{L}_{\mathrm{GW}}}

\newcommand{\minGSGW}{\mathrm{min\text{-}GSGW}}

\title{Min Generalized Sliced Gromov-Wasserstein:\\ A Scalable Path to Gromov–Wasserstein}

\author{
  Ashkan Shahbazi$^{*,1}$ \quad Xinran Liu$^{*,1}$ \quad Ping He$^{1}$ \quad Soheil Kolouri$^{1,2}$ \\[0.5em]
  $^{1}$Department of Computer Science, College of Connected Computing, Vanderbilt University \\
  $^{2}$Department of Electrical and Computer Engineering, Vanderbilt University \\
  \texttt{\{ashkan.shahbazi, xinran.liu, ping.he, soheil.kolouri\}@vanderbilt.edu}
}

\begin{document}
\maketitle

\begingroup
\renewcommand\thefootnote{}
\footnotetext{* Equal contribution.}
\endgroup

\begin{abstract}
We propose min Generalized Sliced Gromov--Wasserstein (\(\minGSGW\)), a sliced
formulation for the Gromov--Wasserstein (GW) problem using expressive
generalized slicers. The key idea is to learn coupled nonlinear slicers that assign compatible push-forward values to both input measures, so that monotone coupling in the projected domain lifts to a transport plan evaluated against the GW objective in the original spaces. The resulting plan induces a GW objective value, and \(\minGSGW\) minimizes this cost directly in the original
spaces. We further show that \(\minGSGW\) is rigid-motion invariant, a crucial
property for geometric matching and shape analysis tasks. Our contributions are
threefold: 1) we introduce generalized slicers into the sliced GW framework,
2) we construct a slicing-based efficient GW transport plan; and 3) we develop an amortized variant that replaces per-instance optimization with a learned slicer for unseen input pairs. We perform experiments on animal mesh matching, horse mesh interpolation, and ShapeNet part transfer. Results show that \(\minGSGW\) produces meaningful geometric correspondences and GW objective values at substantially lower computational cost than existing GW solvers.
\end{abstract}

% keywords can be removed

\section{Introduction}

Modern machine learning increasingly requires comparing objects without a common ambient space, arising in shape analysis, graph matching, biological data integration, and representation comparison, where the relevant signal lies in pairwise intra-domain geometry rather than absolute coordinates. Gromov--Wasserstein (GW) addresses this by comparing metric measure spaces through couplings that preserve internal geometry, making it a central tool for structure-aware matching \citep{journals/focm/Memoli11,pmlr-v48-peyre16}. Despite these strengths, GW remains computationally expensive. In the discrete setting it is a nonconvex quadratic program over couplings, and practical solvers are substantially costlier than classical OT. This has motivated scalable alternatives including entropic, low-rank, semidefinite, and sliced formulations \citep{NEURIPS2019_a9cc6694, pmlr-v162-scetbon22b, chen2023semidefinite, pan2026maxmin}. Sliced GW is especially appealing because it reduces the problem to repeated one-dimensional comparisons admitting $O(n\log n)$ sorting primitives \citep{NEURIPS2019_a9cc6694}. However, this reduction requires care: existing sliced GW methods replace the one-dimensional GW subproblem with monotone matching, which does not solve the one-dimensional GW problem exactly. Counterexamples show that neither the identity nor the anti-identity permutation is optimal for the relevant one-dimensional assignment problem in general \citep{beinert2023assignment,Dumont_2024}. The efficiency of sliced GW therefore comes from limiting the coupling family to those producible by linear projections, not from exactly solving a one-dimensional subproblem \citep{zhang2023duality}. This exposes three coupled limitations of existing sliced GW formulations.
First, the error from monotone matching is tied to the projection family and cannot be eliminated within a purely linear slicing scheme. Second, independently chosen projections for two spaces need not be compatible with a meaningful structural alignment. Third, and perhaps most critically for practical use, the monotone plans produced in the projected space offer no reliability guarantees when lifted back to the original measures: there is no reason in general to expect such a plan to be close to the GW-optimal coupling with respect to the original geometry. As a result, prior sliced GW methods are best understood as alternative distance objectives rather than methods that minimize GW directly. Recent max--min formulations partially address the second issue by coupling slicers adversarially and recovering invariance properties, but they still rely on monotone one-dimensional matchings and do not resolve the plan reliability problem \citep{pan2026maxmin}.

We propose \emph{min Generalized Sliced Gromov--Wasserstein} $(\minGSGW)$,
which replaces linear projections with nonlinear slicers learned through a
shared latent construction with lifting. The key insight is that the limitation
of monotone matching lies not in sorting itself, but in the linear slicer class.
Linear projections can only induce couplings over a narrow family of matchings;
by contrast, nonlinear slicers warp the push-forward ordering so that sorting
the resulting push-forward values induces a coupling in the \emph{original}
metric spaces --- one that can reach regions of the coupling space unavailable
to linear projections. This is the mechanism by which $\minGSGW$ retains $O(n\log n)$ plan
extraction while producing substantially richer couplings than the linear
sliced family. Empirically, these plans frequently achieve GW objective values
competitive with those of more expensive iterative solvers such as Frank--Wolfe
\citep{pmlr-v28-jaggi13} and Sinkhorn \citep{NIPS2013_af21d0c9}, and
sometimes better --- consistent with GW's nonconvexity, where iterative solvers
are not guaranteed to find globally optimal couplings. We first formulate
$\minGSGW$ as a per-instance optimization, then extend it to an amortized
setting where a learned predictor outputs couplings directly for unseen pairs
via a forward pass.

\textbf{Our main contributions are as follows:}
\begin{itemize}
    \item We introduce $\minGSGW$, a generalized sliced GW formulation that couples
    two slicers through a shared map and lifting, inducing transport plans directly
    in the original metric spaces by sorting the resulting push-forward values. The nonlinearity
    of the slicers allows the induced couplings to reach regions of the coupling
    space unavailable to linear projections, yielding lower GW objective values
    within this slicer-induced family.
    \item We establish structural properties of $\minGSGW$, including rigid motion invariance, and optimize the objective with a differentiable sorting relaxation based on LapSum \citep{lapsum2025}.
    \item We develop an amortized variant replacing per-instance optimization with a learned slicer for new input pairs.
    \item We validate on shape matching \citep{10.1145/1015706.1015736}, horse shape interpolation \citep{10.1145/2897824.2925903}, and amortized ShapeNet part segmentation \citep{shapenet2015}, where $\minGSGW$ achieves GW values on par with or better than substantially more expensive iterative solvers.
\end{itemize}

\section{Related Work}

Our work lies at the intersection of sliced GW formulations and sliced OT plan construction. Classical GW solvers use conditional-gradient and entropic local optimization, making the quadratic matching problem practical but requiring iterative dense coupling updates \cite{pmlr-v48-peyre16,10.1145/2897824.2925903,flamary2021pot}. Low-rank, sampled, and relaxation-based methods improve scalability or provide alternative formulations \cite{pmlr-v162-scetbon22b,kerdoncuff2021sampled,chen2023semidefinite,zhang2023duality, chowdhury2021quantizedgromovwasserstein,NEURIPS2019_6e62a992, vay2019fgw}. Sliced GW replaces the full matching problem with one-dimensional projected objectives built from monotone matching \cite{NEURIPS2019_a9cc6694}, though monotone matching does not solve the one-dimensional GW problem exactly \cite{beinert2023assignment}. Max--min formulations address the incompatibility of independently chosen projections and recover stronger invariance properties, but remain distance objectives defined in projected space rather than plan-producing GW minimizations \cite{pan2026maxmin}. A parallel line uses slicing to construct explicit transport plans via monotone one-dimensional matchings, including min-sliced, generalized sliced, expected sliced, and amortized sliced plan formulations \cite{ICLR2025_98db6567,liu2025efficienttransferableoptimaltransport,NEURIPS2023_6f1346ba, rowland2019orthogonal, chapel2026differentiable}. Our method lies at the intersection: like sliced OT plan methods, it constructs couplings from slicer push-forward values via monotone matching; unlike projected sliced GW objectives, it evaluates and optimizes this coupling against the original GW loss in the original metric spaces. $\minGSGW$ thus directly minimizes GW over the family of slicer-induced couplings, yielding an upper bound on GW whose tightness grows with the expressivity of the slicers. The amortized variant uses learning to amortize plan construction across instances rather than to replace the geometric objective.

\section{Preliminaries and Background}
\label{sec:background}

\paragraph{Gromov--Wasserstein.}
Let \((X,d_X,\mu)\) and \((Y,d_Y,\nu)\) be metric measure spaces, with
\(\mu \in \mathcal P_2(X)\) and \(\nu \in \mathcal P_2(Y)\).
Since \(X\) and \(Y\) need not coincide, there is in general no canonical
cross-space cost, and the Wasserstein construction is not directly applicable.
The Gromov--Wasserstein discrepancy instead compares the \emph{internal}
geometries of the two spaces. Its quadratic form is
\begin{equation}
\label{eq:gw}
\GW^2(\mu,\nu)
:=
\inf_{\pi \in \Pi(\mu,\nu)}
\iint_{(X\times Y)^2}
\bigl(d_X(x,x') - d_Y(y,y')\bigr)^2
\,d\pi(x,y)\,d\pi(x',y').
\end{equation}
In the discrete setting, let
\(\mu = \sum_{i=1}^n a_i \delta_{x_i}\) and
\(\nu = \sum_{j=1}^m b_j \delta_{y_j}\)
with \(a\in\Delta_n\), \(b\in\Delta_m\), and define the intra-space distance
matrices \(C^X_{ii'} := d_X(x_i,x_{i'})\) and \(C^Y_{jj'} := d_Y(y_j,y_{j'})\).
The GW loss at a coupling \(\pi\in\Pi(a,b)\) is
\begin{equation}
\label{eq:lgw}
\LGW(\mu,\nu;\pi)
:=
\sum_{i,i',j,j'}
\bigl(C^X_{ii'} - C^Y_{jj'}\bigr)^2\,\pi_{ij}\,\pi_{i'j'},
\end{equation}
and the squared GW distance is
\(\GW^2(\mu,\nu) = \LGW(\mu,\nu;\pi^*)\),
where \(\pi^* := \operatorname*{arg\,min}_{\pi\in\Pi(a,b)}\LGW(\mu,\nu;\pi)\).
This is a nonconvex quadratic program that is NP-hard in general, with a naive \(O(n^4)\) implementation, motivating scalable approximations.

\paragraph{Sliced GW and its limitations.}
With \(X\subset\mathbb{R}^{p}\) and \(Y\subset\mathbb{R}^{q}\),
sliced GW projects to the real line and averages.
Let \(d:=\max(p,q)\) and let
\(\iota_X:\mathbb{R}^{p}\hookrightarrow\mathbb{R}^d\),
\(\iota_Y:\mathbb{R}^{q}\hookrightarrow\mathbb{R}^d\)
be the canonical zero-padding embeddings
(appending \(d-p\) or \(d-q\) zeros, respectively).
For \(\theta\in\mathbb{S}^{d-1}\), define linear projections
\(P_{\theta,X}:=\langle\theta,\iota_X(\cdot)\rangle\)
and \(P_{\theta,Y}:=\langle\theta,\iota_Y(\cdot)\rangle\)
with pushforward measures \((P_{\theta,X})_\#\mu\)
and \((P_{\theta,Y})_\#\nu\) on \(\mathbb{R}\).
The \emph{shared-direction} SGW of \cite{NEURIPS2019_a9cc6694} is
\begin{equation}
\label{eq:sgw_shared}
\operatorname{SGW}_{\mathrm{shared}}(\mu,\nu)
:=
\mathbb{E}_{\theta\sim\sigma}
\Bigl[\LGW\bigl(
  (P_{\theta,X})_\#\mu,\,
  (P_{\theta,Y})_\#\nu;\,
  \pi^\theta_{\mathrm{sort}}
\bigr)\Bigr],
\end{equation}
where \(\pi^\theta_{\mathrm{sort}}\) is the monotone coupling
of the two projected measures (via sorting).

One limitation of this formulation is that it is not rotation-invariant. To mitigate this issue, a naive workaround is to sample projection directions independently for the two spaces. This leads to the \emph{independent-direction} variant: let \(\sigma_X\) and \(\sigma_Y\) denote the uniform distributions on the spheres \(\mathbb{S}^{p-1}\) and \(\mathbb{S}^{q-1}\), respectively, and sample \(\psi\sim\sigma_X\) and \(\phi\sim\sigma_Y\) independently. This yields
\begin{equation}
\label{eq:sgw_indep}
\operatorname{SGW}_{\mathrm{indep}}(\mu,\nu)
:=
\mathbb{E}_{\psi\sim\sigma_X,\,\phi\sim\sigma_Y}
\Bigl[\LGW\bigl(
  (P_\psi)_\#\mu,\,
  (P_\phi)_\#\nu;\,
  \pi^{\psi,\phi}_{\mathrm{sort}}
\bigr)\Bigr],
\end{equation}
where \(P_\psi:=\langle\psi,\cdot\rangle\) and
\(P_\phi:=\langle\phi,\cdot\rangle\) are the standard linear projections
on \(X\) and \(Y\), respectively, and \(\pi^{\psi,\phi}_{\mathrm{sort}}\) is the 1-D monotone coupling between the projections by \(\psi\) and \(\phi\). %The drawback of this workaround is that independent projections break the joint directional correspondence between the two spaces, so the resulting one-dimensional comparisons are less tightly connected to the structural alignments that GW is meant to capture.
The primary limitation is that this formulation violates the identity of indiscernibles.  Pan et al. \cite{pan2026maxmin} address this issue by coupling the two projections through a max-min game. Let \(\Psi:=\{\langle\psi,\cdot\rangle:\psi\in\mathbb{S}^{p-1}\}\)
and \(\Phi:=\{\langle\phi,\cdot\rangle:\phi\in\mathbb{S}^{q-1}\}\)
be the classes of (linear) slicers on \(X\) and \(Y\). Rather than sampling \(\psi\) and \(\phi\) independently and taking the expectation, they solve the following max-min game:
\begin{equation}
\label{eq:maxmin_directed}
\sup_{\psi\in\Psi}\;\inf_{\phi\in\Phi}\;
\LGW((P_\psi)_\#\mu,(P_\phi)_\#\nu;\pi^{\psi,\phi}_{\mathrm{sort}}),
\end{equation}
where the outer player selects \(\psi\) to expose structural discrepancy
and the inner player best-responds with \(\phi\). Symmetrizing yields
\begin{equation}
\label{eq:maxmin_sym}
\max\!\Bigl\{
\sup_{\psi\in\Psi}\inf_{\phi\in\Phi}
  \LGW((P_\psi)_\#\mu,(P_\phi)_\#\nu;\pi^{\psi,\phi}_{\mathrm{sort}}),\;
\sup_{\phi\in\Phi}\inf_{\psi\in\Psi}
  \LGW((P_\phi)_\#\nu, (P_\psi)_\#\mu;\pi^{\phi, \psi}_{\mathrm{sort}})
\Bigr\}.
\end{equation}

While these formulations address some of the shortcomings of naive independent sampling, two fundamental limitations remain common to sliced GW approaches. \textbf{First, the monotone arrangement is in general not the optimal coupling for the one-dimensional GW problem \cite{beinert2023assignment}. Second, they provide no matching/coupling information.}
% -------------------------------------------------------
% OUR METHOD
% -------------------------------------------------------
\section{Method}

\label{sec:method}
We propose \emph{min Generalized Sliced Gromov--Wasserstein} (\(\minGSGW\)), a novel sliced approach to GW that defines a transport plan. It optimizes an objective induced by a transport plan constructed in three steps as shown in ~\ref{fig:teaser}: (1) lift the lower-dimensional measure to the higher-dimensional space via a measurable map, rather than a rigid embedding in \citep{NEURIPS2019_a9cc6694}; (2) apply non-linear slicing to both measures; and (3) compute the transport plan between one-dimensional projections via a flexible non-linear rearrangement, followed by a monotone matching. Importantly, although our method still relies on monotone matching, the use of non-linear rearrangement mitigates the limitation that 1D GW plans cannot be recovered from monotone rearrangements in general.

\begin{figure}[t]
    \centering
    \includegraphics[width=\linewidth]{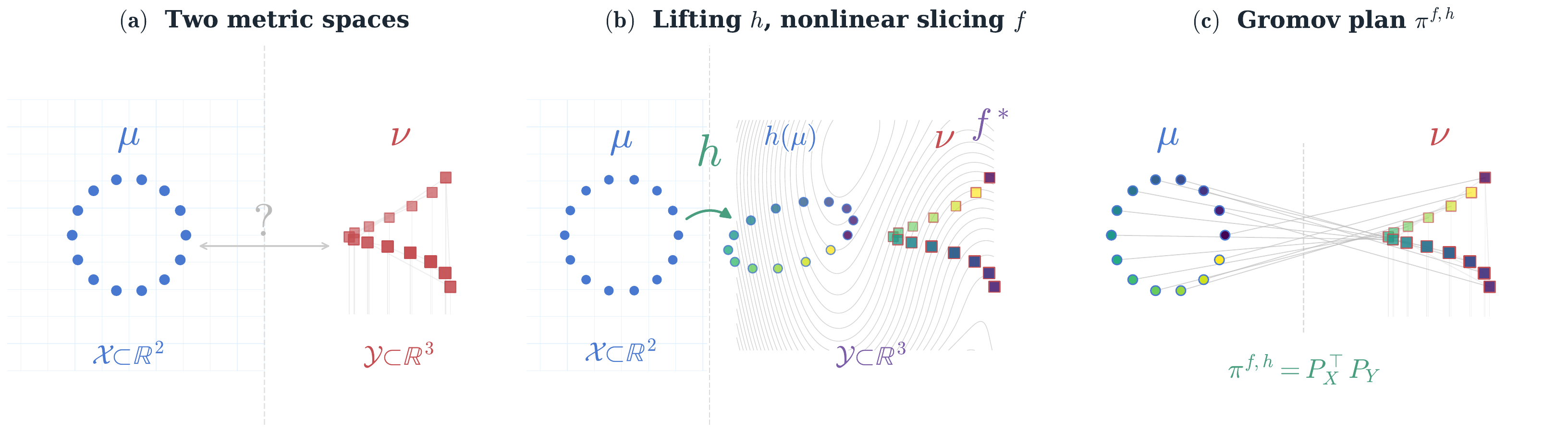}
    \caption{
    \textbf{Overview of \(\minGSGW\).}
    \textbf{(a)}~Two metric measure spaces \(\mu\) on \(\mathcal{X}\subset\mathbb{R}^p\) and \(\nu\) on \(\mathcal{Y}\subset\mathbb{R}^q\) with \(q\!\geq\!p\).
    \textbf{(b)}~A learned lifting \(h:\mathcal{X}\!\to\!\mathcal{Y}\) maps source points into the target domain; a shared nonlinear slicer \(f^*:\mathcal{Y}\!\to\!\mathbb{R}\) then assigns push-forward values to both lifted source points \(h(x_i)\) (circles) and target points \(y_j\) (squares). Level curves of \(f^*\) induce a monotone ordering on both sets.
    \textbf{(c)}~Sorting the two push-forward sequences and matching them into a feasible transport coupling \(\pi^{f,h}\), whose GW objective is minimized over the learned family of slicer-induced couplings.}
    \label{fig:teaser}
\end{figure}

\paragraph{Sliced Gromov--Wasserstein via generalized slicers.}
Let $X\subset\mathbb{R}^p$, $Y\subset \mathbb{R}^q$. Consider discrete measures 
\(\mu=\sum_{i=1}^n a_i\delta_{x_i}\) on $X$ and
\(\nu=\sum_{j=1}^m b_j\delta_{y_j}\) on $Y$, where \(a\in\Delta_n\) and
\(b\in\Delta_m\). Without loss of generality, assume \(p\le q\), so that \(Y\) is the higher-dimensional space. Let \(\mathcal H\) be a class of measurable liftings \(h:X\to Y\), and let $\mathcal{F}$ be the class of slicers \(f:Y\to\mathbb R\) . For \((f,h)\in\mathcal F\times\mathcal H\), define
\[
s_i := (f\circ h)(x_i),\qquad t_j := f(y_j).
\]

Denote by $\pi^*_{f, h}(\mu, \nu)$ the coupling between $\mu$ and $\nu$ that is lifted from the optimal GW coupling between $\{s_i\}_{i=1}^n$ and $\{t_j\}_{j=1}^m$, then we derive an objective

\begin{equation}
    \label{eq: orig obj}
    \inf_{f\in\mathcal{F}, h\in\mathcal{H}}\mathcal{L}_{GW}(\mu, \nu, \pi^*_{f, h})
\end{equation}

However, this formulation does not improve computational efficiency, as it still requires solving a one-dimensional GW problem. To mitigate this issue, we propose to approximate the 1D GW plan by composing a non-linear rearrangement with a monotone coupling.

\paragraph{Approximating the 1D Gromov-Wasserstein solution.}
Instead of solving the 1D GW problem to obtain $\pi^*_{f, h}(\mu, \nu)$, we look for a measurable function $\xi^*:\mathbb{R}\rightarrow\mathbb{R}$ that could be non-linear and achieves

\begin{equation}
    \label{eq:non-linear rearrangement}
    \inf_\xi \mathcal{L}_{GW}((f\circ h)_\#\mu, f_\#\nu, \hat{\pi}_{f, h}^\xi), \qquad\text{(1D problem)}
\end{equation}
for a given pair of $f$ and $h$, where $\hat{\pi}_{f, h}^\xi$ denotes the monotone coupling between 1D measures $\xi_\#((f\circ h)_\#\mu)$ and $\xi_\#(f_\#\nu)$ via sorting. 

% \eva{Justification needed for this approximation}
\begin{proposition}[Representation of the 1D GW plan]
\label{prop: rep 1D GW plan}
Let $\mu=\frac{1}{n}\sum_{i=1}^n \delta_{x_i}, \nu=\frac{1}{m}\sum_{j=1}^m \delta_{y_j}$
be probability measures on \(\mathbb{R}\), with \(n=m\), uniform weights, with the squared Euclidean cost $C^X_{ii'}=d^2_X(x_i, x_{i'})$ and $C^Y_{jj'}=d^2_Y(y_j, y_{j'})$. 
Then there exists an optimal Gromov--Wasserstein plan \(\pi^*\in\Pi(\mu,\nu)\) that is induced by a permutation \(\sigma \in S_n\), i.e.
\[
\pi^* = \frac{1}{n} \sum_{i=1}^n \delta_{(x_i, y_{\sigma(i)})}.
\]

Moreover, there exists a nonlinear map \(\xi:\mathbb{R}\to\mathbb{R}\) such that
the monotone coupling between \(\xi_{\#}\mu\) and \(\xi_{\#}\nu\) induces the
same permutation \(\sigma\). Equivalently,
\[
\pi^*
=
\frac{1}{n}\sum_{i=1}^n \delta_{(x_i,y_{\sigma(i)})}
=
\pi^{\mathrm{mon}}(\xi_{\#}\mu,\xi_{\#}\nu)
\]
after identifying each projected point with its preimage in the original
supports.
\end{proposition}
\begin{proof}
    By \citep{NEURIPS2019_a9cc6694}[Theorem 3.2], the GW problem is equivalent to the Gromov-Monge problem \citep{memoli2018gromov}:
    $$GM_2(\mu, \nu)=
\min_{\sigma \in S_n}
\frac{1}{n^2}
\sum_{i,j}
\left| C^X(x_i, x_j) - C^Y\big(y_{\sigma(i)}, y_{\sigma(j)}\big) \right|^2$$

Therefore, the GW plan is induced by a permutation map $\sigma^*$. Now $\xi:\mathbb{R}\rightarrow\mathbb{R}$ can be constructed by interpolating on the finite set \(\{x_i\}_{i=1}^n \cup \{y_j\}_{j=1}^n\), assigning values that realize

$$\xi(x_1) < \xi(x_2) < \cdots < \xi(x_n); \text{and}$$
$$
    \xi(y_{\sigma^*(1)}) < \xi(y_{\sigma^*(2)}) < \cdots < \xi(y_{\sigma^*(n)}).
$$
Hence $\pi^*=\pi^{\mathrm{mon}}(\xi_{\#}\mu,\xi_{\#}\nu)$.
\end{proof}
Consequently, \eqref{eq:non-linear rearrangement} recovers the 1D GW distance under the assumptions in Proposition \ref{prop: rep 1D GW plan}.

Now, rather than nesting \eqref{eq:non-linear rearrangement} into \eqref{eq: orig obj} as a bilevel optimization problem, we propose to jointly optimize $f, h$ and $\xi$, i.e., the optimization problem becomes 
\begin{equation}
    \label{eq:obj}
    \inf_{f, h, \xi}\mathcal{L}_{GW}(\mu, \nu, \pi^\xi_{f, h})
\end{equation}

where $\pi^\xi_{f, h}$ is the lifted plan from $\hat{\pi}_{f, h}^\xi$. 

\paragraph{Absorbing the 1D rearrangement into the slicer.} Observe that $\pi^\xi_{f, h}$ corresponds to the monotone coupling between $(\xi\circ f\circ h)_\# \mu$ and $(\xi\circ f)_\# \nu$. Since $\xi$ and $f$ only appear through their composition,  we treat $\xi\circ f$ as a single generalized slicer, i.e., the non-linearity of $\xi:\mathbb{R}\rightarrow\mathbb{R}$ can be absorbed into the generalized slicer $f:\mathbb{R}^d\rightarrow\mathbb{R}$. With a slight abuse of notation, we denote this composition $\xi\circ f$ by $f: \mathbb{R}^d\rightarrow\mathbb{R}$ as this relabeling does not affect the formulation under the assumption that such composition keeps the underlying function classes unchanged.

\subsection{min Generalized Sliced Gromov--Wasserstein}
With the above in place, we now introduce the \emph{min Generalized Sliced Gromov--Wasserstein} (\(\minGSGW\)) problem:

\begin{equation}
    \label{eq:obj}
    \minGSGW(\mu,\nu)
:=\inf_{f, h}\mathcal{L}_{GW}(\mu, \nu, \pi^{\text{mon}}_{f, h})
\end{equation}

where $\pi^{\text{mon}}_{f, h}$ denotes the lifted plan from the 1D monotone coupling between $(f\circ h)_\#\mu$ and $f_\#\nu$. See Figure \ref{fig:teaser}.
In the following, we provide some properties of \(\minGSGW\).

% \paragraph{Issue 1 and the role of nonlinearity.}
% Nonlinearity enlarges the restricted score-induced family
% \(\{\pi^{f,h}: f\in\mathcal F,\;h\in\mathcal H\}\subset\Pi(a,b)\), thereby
% increasing the expressivity of the admissible sorted plans. Nonlinear scalar
% maps and liftings can therefore reduce the approximation error introduced by
% restricting GW to score-induced couplings and can tighten the upper bound in
% \eqref{eq:upper_bound_mingsgw_min}. In favorable regimes, this family may
% contain an optimal GW plan, making the bound tight, although no such guarantee
% holds in general.

% \paragraph{Issue 2 is resolved by construction.}
% The scalarizations on \(X\) and \(Y\) are not chosen independently. Instead,
% both scores are induced by the same scalar map \(f:Y\to\mathbb R\), with the
% \(X\)-scores obtained through the lifting \(h:X\to Y\):
% \[
% x_i\mapsto (f\circ h)(x_i),
% \qquad
% y_j\mapsto f(y_j).
% \]
% Thus, the two orderings are coupled by design, avoiding the incompatibility
% that can arise from independently chosen one-dimensional scalarizations.

\begin{proposition}[Rigid motion invariance]
\label{prop:rigid}
Let \(\mathcal E(d)\) denote the Euclidean group acting on \(\mathbb R^d\). Let $X\subset\mathbb{R}^p$ and $Y\subset\mathbb{R}^q$. And let \(\mu=\sum_{i=1}^n a_i\delta_{x_i}\) be a discrete measure on $X$ and
\(\nu=\sum_{j=1}^m b_j\delta_{y_j}\) on $Y$, where \(a\in\Delta_n\) and
\(b\in\Delta_m\). Assume that \(\mathcal F\) and \(\mathcal H\) are stable under rigid motions:
for all \((f,h)\in\mathcal F\times\mathcal H\),
\(g_X\in\mathcal E(p)\), and \(g_Y\in\mathcal E(q)\),
\[
f\circ g_Y^{-1}\in\mathcal F,
\qquad
g_Y\circ h\circ g_X^{-1}\in\mathcal H .
\]
Then
\[
\minGSGW(g_{X\#}\mu,g_{Y\#}\nu)
=
\minGSGW(\mu,\nu).
\]
\end{proposition}

\begin{proof}
We start by showing $\minGSGW(g_{X\#}\mu,g_{Y\#}\nu)\le\minGSGW(\mu,\nu).$

For any given \((f,h)\in\mathcal F\times\mathcal H\), let
\(\tilde f:=f\circ g_Y^{-1}\) and
\(\tilde h:=g_Y\circ h\circ g_X^{-1}\). By stability of $\mathcal{F}$ and $\mathcal{H}$,
\((\tilde f,\tilde h)\in\mathcal F\times\mathcal H\). For any support points \(\tilde x=g_Xx\) and \(\tilde y=g_Yy\) in the pushforward measures $g_{X\#}\mu$ and $g_{Y\#}\nu$, we have $x=g_X^{-1}\tilde{x}$ and $y=g_Y^{-1}\tilde{y}$ by the invertibility of $g_X$ and $g_Y$. Hence
\[
\tilde f(\tilde y)=f(y),
\qquad
\tilde f(\tilde h(\tilde x))=f(h(x)).
\]
% Hence the transformed and original score systems are identical, so the induced
% hard plans, and likewise their LapSum relaxations, agree up to the natural
% relabeling of points. Since rigid motions preserve intra-space distances, the
% corresponding GW objective value is unchanged. Taking the infimum gives one
% inequality; applying the same argument to \(g_X^{-1}\) and \(g_Y^{-1}\) gives
% the reverse inequality. Thus equality holds. The full proof is provided in
% Appendix~\ref{app:rigid-proof}.
That is, \(\tilde f\) and \(\tilde h\) are chosen so that the projected values are preserved pointwise. Hence, the induced one-dimensional monotone coupling is the same in both constructions and the lifted coupling matrices coincide: $\pi^\text{mon}_{\tilde f, \tilde h}=\pi^\text{mon}_{f, h}$. Besides, the pairwise distances inside each space $C^X$ and $C^Y$ are preserved by rigid motions, so $\tilde f, \tilde h$ will result in the same objective value as $f, h$: 
$$\mathcal{L}_{GW}(g_{X\#}\mu,g_{Y\#}\nu, \pi^\text{mon}_{\tilde f, \tilde h})=\mathcal{L}_{GW}(\mu,\nu, \pi^\text{mon}_{f, h}).$$
As this applies to any $(f, h)\in\mathcal{F}\times\mathcal{H}$, we conclude
$$\minGSGW(g_{X\#}\mu,g_{Y\#}\nu)\le\minGSGW(\mu,\nu).$$

Similarly, we can obtain $\minGSGW(\mu,\nu)\le \minGSGW(g_{X\#}\mu,g_{Y\#}\nu)$, and therefore $\minGSGW(\mu,\nu)= \minGSGW(g_{X\#}\mu,g_{Y\#}\nu).$

\end{proof}

\paragraph{Relation to GW.}
Since each \(\pi^{f,h}\) is a feasible coupling for uniform empirical measures,
\(\minGSGW\) is an upper-bound approximation to GW in that setting. Let
\(\pi^*\in\operatorname*{arg\,min}_{\pi\in\Pi(a,b)}
\LGW(\mu,\nu;\pi)\) be an optimal GW coupling. Then, for every
\((f,h)\in\mathcal F\times\mathcal H\),
\(\LGW(\mu,\nu;\pi^*)\le \LGW(\mu,\nu;\pi^{f,h})\). Taking the infimum over
\((f,h)\in\mathcal F\times\mathcal H\) gives
\begin{equation}
\label{eq:upper_bound_mingsgw_min}
\GW^2(\mu,\nu)
=
\LGW(\mu,\nu;\pi^*)
\le
\inf_{f\in\mathcal F,\;h\in\mathcal H}
\LGW\!\bigl(\mu,\nu;\pi^{f,h}\bigr)
=
\minGSGW(\mu,\nu).
\end{equation}
If the restricted family
\(\{\pi^{f,h}: f\in\mathcal F,\;h\in\mathcal H\}\) contains an optimal GW plan,
then the bound is tight.

\subsection{Numerical Implementations}
Let $P_X^{f,h}\in\mathbb{R}^{n\times n}$ and $P_Y^f\in\mathbb{R}^{m\times m}$ be the permutation matrices for $\sigma_X^{f,h}\in S_n$ and $\sigma_Y^f\in S_m$ that sort $s=(s_1,\ldots,s_n)$ and $t=(t_1,\ldots,t_m)$ in nondecreasing order. The transport plan is obtained by monotone matching of the sorted values. For unequal cardinalities, let $T_{n,m}\in\mathbb{R}_+^{n\times m}$ be the canonical monotone interpolation matrix with entries
\[
(T_{n,m})_{ij}:=\lambda\!\left(\Bigl[\tfrac{i-1}{n},\tfrac{i}{n}\Bigr]\cap\Bigl[\tfrac{j-1}{m},\tfrac{j}{m}\Bigr]\right),
\]
where $\lambda$ is the one-dimensional Lebesgue measure. Then $T_{n,m}\mathbf{1}_m=\tfrac{1}{n}\mathbf{1}_n$ and $T_{n,m}^\top\mathbf{1}_n=\tfrac{1}{m}\mathbf{1}_m$, so $T_{n,m}$ preserves uniform marginals while interpolating monotonically between sorted empirical grids. The hard plan induced by $(f,h)$ is
\begin{equation}\label{eq:hard_plan_general}
\pi^{f,h}:=(P_X^{f,h})^\top T_{n,m}\,P_Y^f,
\end{equation}
and feasibility follows immediately: $\pi^{f,h}\mathbf{1}_m=\tfrac{1}{n}\mathbf{1}_n$ and $(\pi^{f,h})^\top\mathbf{1}_n=\tfrac{1}{m}\mathbf{1}_m$, so $\pi^{f,h}\in\Pi(\tfrac{1}{n}\mathbf{1}_n,\tfrac{1}{m}\mathbf{1}_m)$. When $n=m$, $T_{n,n}=\tfrac{1}{n}I_n$ so \eqref{eq:hard_plan_general} simplifies to $\pi^{f,h}=\tfrac{1}{n}(P_X^{f,h})^\top P_Y^f$. The scores determine the coupling, while the loss is evaluated on the original intra-space costs $C^X$ and $C^Y$. Our discrepancy is
\begin{equation}\label{eq:min-GSGW}
\minGSGW(\mu,\nu):=\inf_{f\in\mathcal{F},\,h\in\mathcal{H}}\LGW\!\bigl(\mu,\nu;\pi^{f,h}\bigr),
\end{equation}
minimizing the GW objective over score-induced couplings rather than all of $\Pi(a,b)$.

\paragraph{Differentiable relaxation.}
Since $P_X^{f,h}$ and $P_Y^f$ are not differentiable, during training we replace them with soft sorting matrices $P_{X,\tau}^{f,h}\in[0,1]^{n\times n}$ and $P_{Y,\tau}^f\in[0,1]^{m\times m}$ via a differentiable relaxation such as LapSum~\citep{lapsum2025}; as $\tau$ is annealed these converge to the hard permutations. The soft plan is $\pi_\tau^{f,h}:=(P_{X,\tau}^{f,h})^\top T_{n,m}\,P_{Y,\tau}^f$, reducing to $\pi_\tau^{f,h}=\tfrac{1}{n}(P_{X,\tau}^{f,h})^\top P_{Y,\tau}^f$ when $n=m$. If $P_{X,\tau}^{f,h}$ and $P_{Y,\tau}^f$ are doubly stochastic, $\pi_\tau^{f,h}$ inherits the correct uniform marginals. During training we optimize
\begin{equation}\label{eq:soft-minGSGW}
\inf_{f\in\mathcal{F},\,h\in\mathcal{H}}\LGW\!\bigl(\mu,\nu;\pi_\tau^{f,h}\bigr),
\end{equation}
and at inference we use the hard plan $\pi^{f,h}$.

\subsection{Computational complexity}

Our method induces a transport plan by sorting slicer push-forward values,
costing $O(n\log n)$ excluding the slicer forward pass, with the same
construction used during training via the differentiable LapSum relaxation.
Sliced baselines range from $O(Ln\log n)$ (SGW) to $O(L^2 n\log n)$ (MSGW),
while dense GW solvers (POT-GW, entropic GW) incur cubic per-iteration cost due
to the quadratic GW tensor contraction. Figure~\ref{fig:runtime_complexity_and_benchmark}
reports wall-clock runtimes on an RTX A6000 averaged over $10$ runs, confirming
that our method scales most favorably among such methods after its training.

\begin{figure}[t]
\centering

\begin{minipage}[t]{0.42\linewidth}
\vspace{0pt}
\centering
\small
\setlength{\tabcolsep}{5pt}
\renewcommand{\arraystretch}{1.15}
\begin{tabular}{lc}
\toprule
\textbf{Method} & \textbf{Runtime complexity} \\
\midrule
POT-GW 
& NP-hard; $\mathcal{O}(n^4); \mathcal{O}(n^3)$ \\

Sinkhorn 
& $\mathcal{O}(n^4); \mathcal{O}(n^3)$ \\

SGW 
& $\mathcal{O}(L n\log n)$ \\

MSGW 
& $\mathcal{O}(L^2 n\log n)$ \\

RI-SGW 
& $\mathcal{O}\!\left(n_{\mathrm{iter}}\!\left(Ln(p+\log n)+p^3\right)\right)$ \\

\rowcolor{green!20}
\textbf{Ours} 
& $\mathbf{\mathcal{O}(n\log n)}$ \\
\bottomrule
\end{tabular}

\vspace{2pt}
\subcaption*{Runtime complexity and empirical scaling. Wall-clock times are averaged over 10 runs for \(N\in\{100:5000\}\). MSGW runs out of memory beyond $N=1000$, while our method has the lowest complexity and favorable scaling.}
\label{fig:runtime_complexity_table}
\end{minipage}
\hfill
\begin{minipage}[t]{0.56\linewidth}
\vspace{0pt}
\centering
\includegraphics[width=\linewidth]{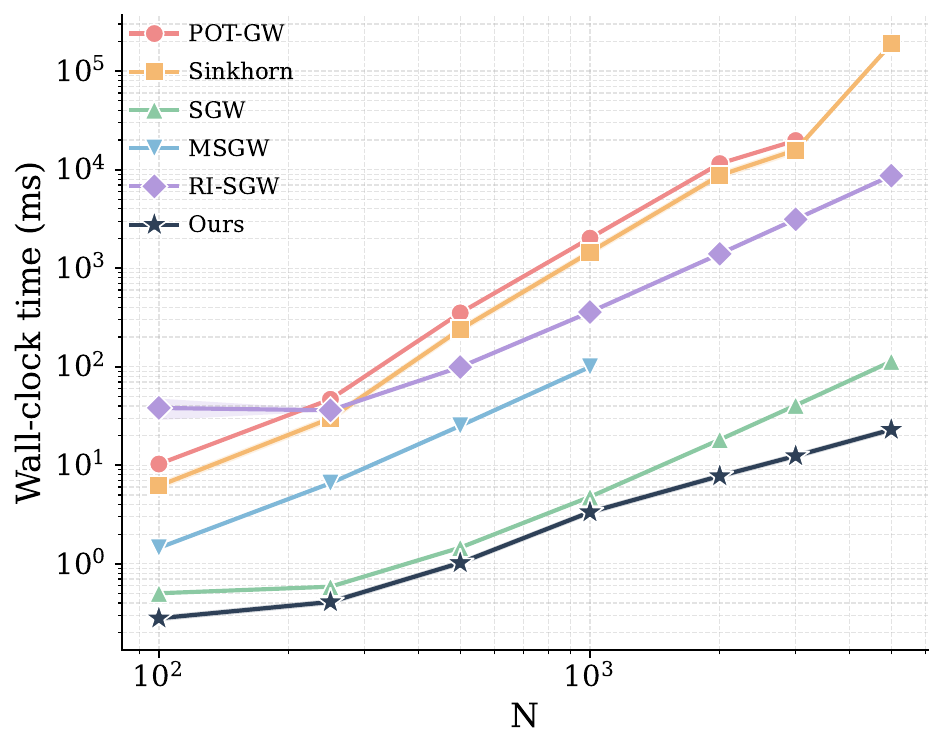}

\vspace{2pt}
% \subcaption{Measured wall-clock time.}
\label{fig:runtime_benchmark}
\end{minipage}

\captionsetup{labelformat=empty,textformat=empty}
\caption{Runtime comparison.}
\label{fig:runtime_complexity_and_benchmark}
\captionsetup{labelformat=default,textformat=simple}
\vspace{-0.8in}
\end{figure}

\section{Experiments}

\label{sec:experiments}

We evaluate \(\minGSGW\) on animal mesh matching, horse shape interpolation, and amortized ShapeNet matching. The results are averaged over three runs. Toy correspondence experiments are deferred to Appendix~\ref{app:toy_correspondences}, and implementation details, hyperparameters, and reproducibility information are given in Appendix~\ref{app:exp_details}. Full ablations over linear vs.\ nonlinear slicers and independent vs.\
dependent slicers are deferred to Appendix~\ref{app:ablation}; the soft
sorting temperature is annealed during training across all variants.

\subsection{Realistic Shape Matching}

Following \citep{chen2023semidefinite}, we evaluate on a realistic shape matching benchmark derived
from \citep{10.1145/1015706.1015736}, where each shape is represented by a geodesic distance matrix
with uniform weights. Since each method optimizes a different surrogate of the GW objective,, we instead evaluate via \textit{geodesic error} --- the Princeton
protocol metric \citep{10.1145/2897824.2925903}: the mean normalized geodesic distance between
predicted and ground-truth correspondences, which directly measures geometric accuracy independent
of each method's internal objective. We compare against POT, Sinkhorn, LR-GW
\citep{pmlr-v162-scetbon22b}, SDP-GW \citep{chen2023semidefinite}, and SaGroW
\citep{kerdoncuff2021sampled}, reporting geodesic error and runtime in Table~\ref{tab:gw_pot_ours}.
\(\minGSGW\) achieves the lowest geodesic error on all six pairs while remain fast; see also Figure~\ref{fig:shape_correspondence} in
Appendix~\ref{app:exp_details}.

\begin{table}[ht!]
\centering
\scriptsize
\caption{Geodesic error (Geo.) and forward runtime in seconds (Time) on realistic mesh pairs
(all $\downarrow$). H, E, C denote Horse, Elephant, Cat.
LR-GW: \citep{pmlr-v162-scetbon22b}; SDP-GW: \citep{chen2023semidefinite}; SaGroW: \citep{kerdoncuff2021sampled}.}
\setlength{\tabcolsep}{5pt}
\renewcommand{\arraystretch}{1.08}
\resizebox{0.95\linewidth}{!}{%
\begin{tabular}{lrr|rr|rr|rr|rr|rr}
\toprule
& \multicolumn{2}{c|}{H--H}
& \multicolumn{2}{c|}{E--E}
& \multicolumn{2}{c|}{C--C}
& \multicolumn{2}{c|}{H--E}
& \multicolumn{2}{c|}{C--H}
& \multicolumn{2}{c}{C--E} \\
\cmidrule(lr){2-3}\cmidrule(lr){4-5}\cmidrule(lr){6-7}
\cmidrule(lr){8-9}\cmidrule(lr){10-11}\cmidrule(lr){12-13}
Method
& Geo. & Time & Geo. & Time & Geo. & Time
& Geo. & Time & Geo. & Time & Geo. & Time \\
\midrule
POT
& 0.112 & 1.82 & 0.103 & 1.91 & 0.079 & 1.76
& 0.181 & 1.88 & 0.208 & 1.80 & 0.189 & 1.86 \\
Sink.
& 0.178 & 0.41 & 0.194 & 0.43 & 0.138 & 0.39
& 0.247 & 0.42 & 0.281 & 0.40 & 0.236 & 0.41 \\
LR-GW
& 0.143 & 0.87 & 0.157 & 0.94 & 0.103 & 0.81
& 0.213 & 0.89 & 0.241 & 0.84 & 0.207 & 0.91 \\
SDP-GW
& 0.128 & 15.7 & 0.134 & 17.4 & 0.091 & 13.6
& 0.196 & 16.3 & 0.224 & 14.7 & 0.193 & 16.6 \\
SaGroW
& 0.167 & 0.31 & 0.182 & 0.34 & 0.124 & 0.28
& 0.231 & 0.32 & 0.264 & 0.30 & 0.221 & 0.33 \\
\midrule
Ours
& \textbf{0.079} & \textbf{0.08}
& \textbf{0.091} & \textbf{0.08}
& \textbf{0.058} & \textbf{0.07}
& \textbf{0.138} & \textbf{0.08}
& \textbf{0.162} & \textbf{0.07}
& \textbf{0.124} & \textbf{0.08} \\
\bottomrule
\end{tabular}}
\label{tab:gw_pot_ours}
\end{table}

\subsection{Horse mesh interpolation.}

We test whether learned couplings support downstream geometry transfer via
barycentric interpolation between consecutive horse meshes, comparing POT GW,
MSGW, and our method. Figure~\ref{fig:horse_ot_barycentric_interp} shows that,
although POT achieves lower GW objective values, our plans still yield plausible
dense correspondences and smooth intermediate deformations. This highlights a
key distinction: global GW optimality does not imply geometric usefulness, and
the slicer-induced plan family is expressive enough to preserve large-scale
structure for interpolation.

\begin{figure*}[t]
    \centering
    \includegraphics[width=\textwidth]{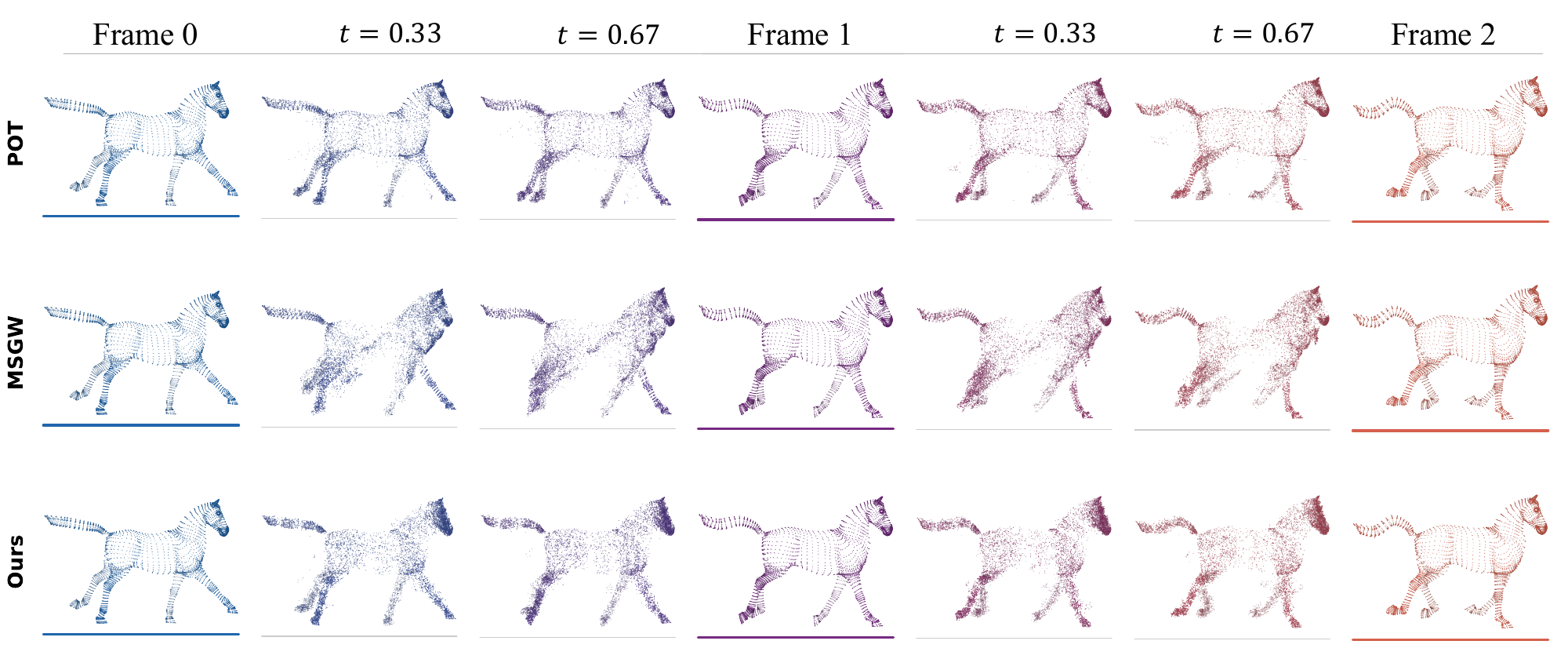}
    \caption{
        \textbf{Horse mesh interpolation from GW couplings.}
        OT barycentric interpolation between consecutive horse meshes at $t=0.33$ and $t=0.67$.
        POT achieves lower GW values $(8.257\times10^{-3}$ for $0\!\to\!1$, $9.714\times10^{-3}$ for $1\!\to\!2)$,
        yet our method still produces smooth, geometrically coherent deformations
        $(2.884\times10^{-2}$ for $0\!\to\!1$, $4.622\times10^{-2}$ for $1\!\to\!2)$.}
    \label{fig:horse_ot_barycentric_interp}
\end{figure*}

\subsection{Amortized Unsupervised Matching for ShapeNet Part Segmentation}

\label{sec:shapenet_amortized}
We study amortized unsupervised matching on ShapeNet part segmentation \citep{shapenet2015}.
Given two point clouds \(X=\{x_i\}_{i=1}^n\) and \(Y=\{y_j\}_{j=1}^m\),
we train a neural matcher that predicts, in a single forward pass, the
push-forward values of \(\minGSGW\) and the induced soft coupling
\(\pi_\tau^\theta(X,Y)\in\Pi(\mu,\nu)\).
Training uses no part labels; part annotations enter only at test time
to report label-transfer accuracy.
\paragraph{Structural constraints.}
For the amortized plan \(G_\theta(X,Y)\) to be a well-formed slicer-induced
coupling, it must satisfy three properties that mirror the non-amortized
formulation (w.l.o.g we assume $d_X = d_y$ here):
\begin{equation}
\label{eq:amortized_constraints}
G_\theta(X,X)=\tfrac{1}{n}I_n,
\qquad
G_\theta(Y,X)=G_\theta(X,Y)^\top,
\qquad
G_\theta(TX,TY)=G_\theta(X,Y)
\quad\forall\,T\in\mathcal{E}(d),
\end{equation}
together with permutation equivariance
\(G_\theta(PX,QY)=P\,G_\theta(X,Y)\,Q^\top\)
for all permutation matrices \(P,Q\).
The identity constraint enforces that a shape matched against itself
recovers the diagonal plan.
Symmetry ensures the matching is consistent under argument swap.
Rigid-motion invariance means the coupling depends only on the intrinsic
geometry of the two shapes, not their ambient orientation.
The architecture is designed so that each constraint is satisfied by
construction rather than by regularization, as we describe next.

\paragraph{Intrinsic tokenisation.}
To preserve rigid-motion invariance at the input level, each point is
encoded by its sorted squared-distance profile within the same point
cloud.  Concretely, for \(x_i\in X\) define
\(\mathcal{D}_i^X=\mathrm{sort}\bigl(\{\|x_i-x_k\|^2:k\neq i\}\bigr)\in\mathbb{R}^{n-1}\),
and set \(\phi_i^X=\rho(\mathcal{D}_i^X)\) where \(\rho\) is a shared
MLP encoder; analogously for \(Y\).  Sorting provides a canonical
ordering that makes each token invariant to permutations of the
remaining points, so the token collections
\(\Phi^X=(\phi_1^X,\dots,\phi_n^X)\) and
\(\Phi^Y=(\phi_1^Y,\dots,\phi_m^Y)\) are rigid-motion invariant and
permutation equivariant by construction.

\paragraph{Push-forward prediction and plan construction.}
Following the coupled-scalarization structure of \(\minGSGW\), the
amortized model maps \((\Phi^X,\Phi^Y)\) to push-forward values
\(s^\theta(X,Y)\in\mathbb{R}^n\) and \(t^\theta(X,Y)\in\mathbb{R}^m\)
via a shared-weight transformer encoder and a cross-attending push-forward head
(architecture details in Appendix~\ref{app:arch}).
The push-forward values induce the soft plan \(\pi_\tau^\theta\) during training
and the hard plan \(\pi^\theta\) at inference, exactly as in
Section~\ref{sec:method}.

\paragraph{Training objective.}
Although the coupling is constructed through the \(\minGSGW\) slicer
construction, we evaluate it under the \emph{fused} GW objective:
\[
\mathcal{L}_{\mathrm{FGW}}(\pi)
\;=\;
(1-\lambda)\!\!\sum_{i,i',j,j'}
\bigl(C_X(i,i')-C_Y(j,j')\bigr)^2\pi_{ij}\pi_{i'j'}
+\lambda\sum_{i,j}\|\phi_i^X-\phi_j^Y\|_2^2\,\pi_{ij},
\]
The model is trained end-to-end by minimizing the expected loss:
\[
\min_\theta\;
\mathbb{E}_{(\mu,\nu)}
\Bigl[
  \mathcal{L}_{\mathrm{FGW}}\!\bigl(\pi_\tau^\theta(\mu,\nu)\bigr)
\Bigr].
\]

\begin{figure*}[th!]
    \centering
    \includegraphics[width=\textwidth]{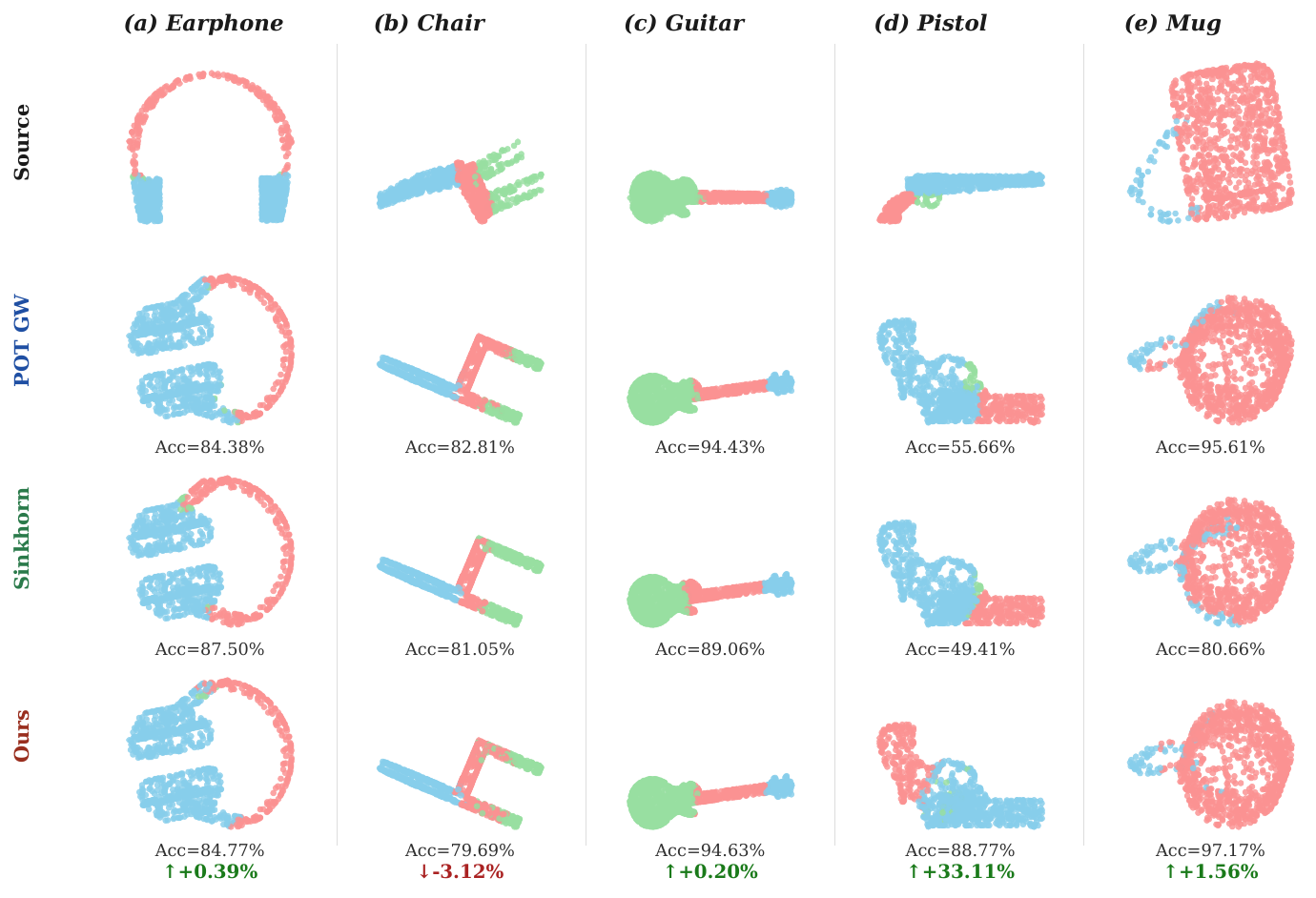}
    \caption{
    \textbf{Qualitative correspondences on ShapeNet.}
    Each column shows one object category, while rows show the source shape, POT GW, GW Sinkhorn, and our method.
    Source colors are propagated to the target via the estimated coupling.
    Bottom annotations report part-matching accuracy, and the last row also reports the accuracy difference between our method and POT GW.
    \vspace{-0.1in}}
    \label{fig:shapenet_qualitative_transposed}
\end{figure*}

\begin{table}[th!]
\centering
\caption{Mean part label transfer accuracy and forward time per pair at varying point resolutions $N$.}
\label{tab:shapenet_resolution}
\small
\setlength{\tabcolsep}{3.5pt}
\begin{tabular}{lcccccc}
\toprule
& \multicolumn{2}{c}{$N=256$} & \multicolumn{2}{c}{$N=512$} & \multicolumn{2}{c}{$N=1024$} \\
\cmidrule(lr){2-3}\cmidrule(lr){4-5}\cmidrule(lr){6-7}
Method & Acc. (\%) $\uparrow$ & Fwd. (ms) $\downarrow$ & Acc. (\%) $\uparrow$ & Fwd. (ms) $\downarrow$ & Acc. (\%) $\uparrow$ & Fwd. (ms) $\downarrow$ \\
\midrule
POT GW & 68.9 & 75.3 & 69.8 & 375.7 & 73.5 & 886.2 \\
Sinkhorn ($\varepsilon = 0.05$) & 66.1 & 164.7 & 66.6 & 361.1 & 71.3 & 843.1 \\
Sinkhorn ($\varepsilon = 0.5$)  & 64.3 & 91.7 & 65.5 & 370.2 & 68.7 & 821.9 \\
Sinkhorn ($\varepsilon = 1$)    & 62.8 & 72.3 & 64.7 & 303.3 & 67.4 & 790.2 \\
\midrule
\textbf{Ours} & \textbf{78.2} & \textbf{2.66} & \textbf{77.5} & \textbf{4.3} & \textbf{74.9} & \textbf{9.4} \\
\bottomrule
\end{tabular}

\end{table}

\section{Conclusion}

We introduced $\minGSGW$, a scalable GW formulation that induces explicit
transport plans by sorting slicer push-forward values, bypassing iterative
coupling updates at inference. By coupling the two slicers through a shared map
and lifting, the method avoids independently chosen projections and minimizes
the original GW loss over a richer family of monotone couplings than linear
slicing affords, yielding a valid upper bound on GW with rigid-motion invariance
under natural stability assumptions. The framework supports both per-instance
optimization and amortized feed-forward matching. Across mesh correspondence,
shape interpolation, and ShapeNet part segmentation, $\minGSGW$ produces
geometrically meaningful couplings with competitive GW objectives while scaling
substantially more favorably than classical plan-producing solvers. These results
suggest that expressive slicer-induced monotone couplings offer a practical path
toward structure-aware matching at resolutions where dense GW optimization
becomes prohibitive.

\paragraph{Broader Impact.} This work advances scalable approximations to Gromov--Wasserstein distance, with potential applications in shape analysis, molecular alignment, and multi-modal data integration. As with any method that reduces the cost of large-scale geometric matching, broader adoption may introduce risks in high-stakes domains, where approximation error could have downstream consequences.

\clearpage

\section*{Acknowledgment}
SK acknowledges support from the NSF CAREER Award No.\ 2339898 and AS acknowledges support from Lambda Labs through a Lambda Cloud Research Credit award.

\bibliographystyle{plain}
\bibliography{references}

@misc{
pan2026maxmin,
title={Max-Min Sliced Gromov-Wasserstein},
author={Wen-Xin Pan and Isabel Haasler and Hei Victor Cheng},
year={2026},
url={https://openreview.net/forum?id=yvbWfcnt0U}
}

@article{memoli2018gromov,
  title={Gromov-Monge quasi-metrics and distance distributions},
  author={M{\'e}moli, Facundo and Needham, Tom},
  journal={arXiv},
  volume={2018},
  year={2018}
}

@inproceedings{NEURIPS2019_a9cc6694,
 author = {Titouan, Vayer and Flamary, R\'{e}mi and Courty, Nicolas and Tavenard, Romain and Chapel, Laetitia},
 booktitle = {Advances in Neural Information Processing Systems},
 editor = {H. Wallach and H. Larochelle and A. Beygelzimer and F. d\textquotesingle Alch\'{e}-Buc and E. Fox and R. Garnett},
 pages = {},
 publisher = {Curran Associates, Inc.},
 title = {Sliced Gromov-Wasserstein},
 url = {https://proceedings.neurips.cc/paper_files/paper/2019/file/a9cc6694dc40736d7a2ec018ea566113-Paper.pdf},
 volume = {32},
 year = {2019}
}

@article{beinert2023assignment,
  title={On assignment problems related to Gromov--Wasserstein distances on the real line},
  author={Beinert, Robert and Heiss, Cosmas and Steidl, Gabriele},
  journal={SIAM Journal on Imaging Sciences},
  volume={16},
  number={2},
  pages={1028--1032},
  year={2023},
  publisher={SIAM}
}

@article{journals/focm/Memoli11,
  author = {Mémoli, Facundo},
  ee = {http://dx.doi.org/10.1007/s10208-011-9093-5},
  journal = {Foundations of Computational Mathematics},
  number = 4,
  pages = {417-487},
  title = {Gromov-Wasserstein Distances and the Metric Approach to Object Matching.},
  url = {http://dblp.uni-trier.de/db/journals/focm/focm11.html#Memoli11},
  volume = 11,
  year = 2011
}

@InProceedings{pmlr-v48-peyre16,
  title = 	 {Gromov-Wasserstein Averaging of Kernel and Distance Matrices},
  author = 	 {Peyré, Gabriel and Cuturi, Marco and Solomon, Justin},
  booktitle = 	 {Proceedings of The 33rd International Conference on Machine Learning},
  pages = 	 {2664--2672},
  year = 	 {2016},
  editor = 	 {Balcan, Maria Florina and Weinberger, Kilian Q.},
  volume = 	 {48},
  series = 	 {Proceedings of Machine Learning Research},
  address = 	 {New York, New York, USA},
  month = 	 {20--22 Jun},
  publisher =    {PMLR},
  url = 	 {https://proceedings.mlr.press/v48/peyre16.html}
}

@InProceedings{pmlr-v162-scetbon22b,
  title = 	 {Linear-Time Gromov {W}asserstein Distances using Low Rank Couplings and Costs},
  author =       {Scetbon, Meyer and Peyr{\'e}, Gabriel and Cuturi, Marco},
  booktitle = 	 {Proceedings of the 39th International Conference on Machine Learning},
  pages = 	 {19347--19365},
  year = 	 {2022},
  editor = 	 {Chaudhuri, Kamalika and Jegelka, Stefanie and Song, Le and Szepesvari, Csaba and Niu, Gang and Sabato, Sivan},
  volume = 	 {162},
  series = 	 {Proceedings of Machine Learning Research},
  month = 	 {17--23 Jul},
  publisher =    {PMLR},
  url = 	 {https://proceedings.mlr.press/v162/scetbon22b.html}
}

@inproceedings{
	chen2023semidefinite,
	title={Semidefinite Relaxations of the Gromov-Wasserstein Distance},
	author={Chen, Junyu and Nguyen, Binh and Koh, Shang Hui and Soh, Yong Sheng},
	booktitle={The Thirty-eighth Annual Conference on Neural Information Processing Systems},
	year={2024},
	url={https://openreview.net/forum?id=rM3FFH1mqk}
}

@article{Dumont_2024,
   title={On the Existence of Monge Maps for the Gromov–Wasserstein Problem},
   volume={25},
   ISSN={1615-3383},
   url={http://dx.doi.org/10.1007/s10208-024-09643-0},
   DOI={10.1007/s10208-024-09643-0},
   number={2},
   journal={Foundations of Computational Mathematics},
   publisher={Springer Science and Business Media LLC},
   author={Dumont, Théo and Lacombe, Théo and Vialard, François-Xavier},
   year={2024},
   month=feb, pages={463–510} }

@inproceedings{
zhang2023duality,
title={Duality and Sample Complexity for the Gromov-Wasserstein Distance},
author={Zhengxin Zhang and Ziv Goldfeld and Youssef Mroueh and Bharath Sriperumbudur},
booktitle={NeurIPS 2023 Workshop Optimal Transport and Machine Learning},
year={2023},
url={https://openreview.net/forum?id=RYPNAOWANk}
}

@inproceedings{lapsum2025,
  title={LapSum - One Method to Differentiate Them All: Ranking, Sorting and Top-k Selection},
  author={\L{}ukasz Struski and Micha\l{} B. Bednarczyk and Igor T. Podolak and Jacek Tabor},
  booktitle={The International Conference on Machine Learning (ICML) 2025},
  year={2025}
}

@InProceedings{pmlr-v28-jaggi13,
  title = 	 {Revisiting {Frank-Wolfe}: Projection-Free Sparse Convex Optimization},
  author = 	 {Jaggi, Martin},
  booktitle = 	 {Proceedings of the 30th International Conference on Machine Learning},
  pages = 	 {427--435},
  year = 	 {2013},
  editor = 	 {Dasgupta, Sanjoy and McAllester, David},
  volume = 	 {28},
  number =       {1},
  series = 	 {Proceedings of Machine Learning Research},
  address = 	 {Atlanta, Georgia, USA},
  month = 	 {17--19 Jun},
  publisher =    {PMLR},
  url = 	 {https://proceedings.mlr.press/v28/jaggi13.html}
}

@inproceedings{NIPS2013_af21d0c9,
 author = {Cuturi, Marco},
 booktitle = {Advances in Neural Information Processing Systems},
 editor = {C.J. Burges and L. Bottou and M. Welling and Z. Ghahramani and K.Q. Weinberger},
 pages = {},
 publisher = {Curran Associates, Inc.},
 title = {Sinkhorn Distances: Lightspeed Computation of Optimal Transport},
 url = {https://proceedings.neurips.cc/paper_files/paper/2013/file/af21d0c97db2e27e13572cbf59eb343d-Paper.pdf},
 volume = {26},
 year = {2013}
}

@techreport{shapenet2015,
  title       = {{ShapeNet: An Information-Rich 3D Model Repository}},
  author      = {Chang, Angel X. and Funkhouser, Thomas and Guibas, Leonidas and Hanrahan, Pat and Huang, Qixing and Li, Zimo and Savarese, Silvio and Savva, Manolis and Song, Shuran and Su, Hao and Xiao, Jianxiong and Yi, Li and Yu, Fisher},
  number      = {arXiv:1512.03012 [cs.GR]},
  institution = {Stanford University --- Princeton University --- Toyota Technological Institute at Chicago},
  year        = {2015}
}

@article{kerdoncuff2021sampled,
  title   = {Sampled Gromov Wasserstein},
  author  = {Kerdoncuff, Tanguy and Emonet, R{\'e}mi and Sebban, Marc},
  journal = {Machine Learning},
  volume  = {110},
  number  = {8},
  pages   = {2151--2186},
  year    = {2021},
  doi     = {10.1007/s10994-021-06035-1},
  publisher = {Springer}
}

@inproceedings{NEURIPS2019_6e62a992,
 author = {Xu, Hongteng and Luo, Dixin and Carin, Lawrence},
 booktitle = {Advances in Neural Information Processing Systems},
 editor = {H. Wallach and H. Larochelle and A. Beygelzimer and F. d\textquotesingle Alch\'{e}-Buc and E. Fox and R. Garnett},
 pages = {},
 publisher = {Curran Associates, Inc.},
 title = {Scalable Gromov-Wasserstein Learning for Graph Partitioning and Matching},
 url = {https://proceedings.neurips.cc/paper_files/paper/2019/file/6e62a992c676f611616097dbea8ea030-Paper.pdf},
 volume = {32},
 year = {2019}
}

@article{10.1145/1015706.1015736,
author = {Sumner, Robert W. and Popovi\'{c}, Jovan},
title = {Deformation transfer for triangle meshes},
year = {2004},
issue_date = {August 2004},
publisher = {Association for Computing Machinery},
address = {New York, NY, USA},
volume = {23},
number = {3},
issn = {0730-0301},
url = {https://doi.org/10.1145/1015706.1015736},
doi = {10.1145/1015706.1015736},
journal = {ACM Trans. Graph.},
month = aug,
pages = {399–405},
numpages = {7}
}

@article{10.1145/2897824.2925903,
author = {Solomon, Justin and Peyr\'{e}, Gabriel and Kim, Vladimir G. and Sra, Suvrit},
title = {Entropic metric alignment for correspondence problems},
year = {2016},
issue_date = {July 2016},
publisher = {Association for Computing Machinery},
address = {New York, NY, USA},
volume = {35},
number = {4},
issn = {0730-0301},
url = {https://doi.org/10.1145/2897824.2925903},
doi = {10.1145/2897824.2925903},
journal = {ACM Trans. Graph.},
month = jul,
articleno = {72},
numpages = {13}
}

@misc{liu2025efficienttransferableoptimaltransport,
      title={Efficient Transferable Optimal Transport via Min-Sliced Transport Plans}, 
      author={Xinran Liu and Elaheh Akbari and Rocio Diaz Martin and Navid NaderiAlizadeh and Soheil Kolouri},
      year={2025},
      eprint={2511.19741},
      archivePrefix={arXiv},
      primaryClass={cs.CV},
      url={https://arxiv.org/abs/2511.19741}, 
}

@inproceedings{ICLR2025_98db6567,
 author = {Liu, Xinran and Diaz Martin, Rocio and Bai, Yikun and Shahbazi, Ashkan and Thorpe, Matthew and Aldroubi, Akram and Kolouri, Soheil},
 booktitle = {International Conference on Learning Representations},
 editor = {Y. Yue and A. Garg and N. Peng and F. Sha and R. Yu},
 pages = {60948--60971},
 title = {Expected Sliced Transport Plans},
 url = {https://proceedings.iclr.cc/paper_files/paper/2025/file/98db6567a141db93b3cdeb177da8ab37-Paper-Conference.pdf},
 volume = {2025},
 year = {2025}
}

@article{flamary2021pot,
  author  = {R{\'e}mi Flamary and Nicolas Courty and Alexandre Gramfort and Mokhtar Z. Alaya and Aur{\'e}lie Boisbunon and Stanislas Chambon and Laetitia Chapel and Adrien Corenflos and Kilian Fatras and Nemo Fournier and L{\'e}o Gautheron and Nathalie T.H. Gayraud and Hicham Janati and Alain Rakotomamonjy and Ievgen Redko and Antoine Rolet and Antony Schutz and Vivien Seguy and Danica J. Sutherland and Romain Tavenard and Alexander Tong and Titouan Vayer},
  title   = {POT: Python Optimal Transport},
  journal = {Journal of Machine Learning Research},
  year    = {2021},
  volume  = {22},
  number  = {78},
  pages   = {1-8},
  url     = {http://jmlr.org/papers/v22/20-451.html}
}

@misc{chowdhury2021quantizedgromovwasserstein,
      title={Quantized Gromov-Wasserstein}, 
      author={Samir Chowdhury and David Miller and Tom Needham},
      year={2021},
      eprint={2104.02013},
      archivePrefix={arXiv},
      primaryClass={cs.LG},
      url={https://arxiv.org/abs/2104.02013}, 
}

@InProceedings{vay2019fgw,
  title =    {Optimal Transport for structured data with application on graphs},
  author =   {Titouan, Vayer and Courty, Nicolas and Tavenard, Romain and Laetitia, Chapel and Flamary, R{\'e}mi},
  booktitle =    {Proceedings of the 36th International Conference on Machine Learning},
  pages =    {6275--6284},
  year =   {2019},
  editor =   {Chaudhuri, Kamalika and Salakhutdinov, Ruslan},
  volume =   {97},
  series =   {Proceedings of Machine Learning Research},
  address =    {Long Beach, California, USA},
  month =    {09--15 Jun},
  publisher =    {PMLR},
  url =    {http://proceedings.mlr.press/v97/titouan19a.html}
}

@inproceedings{NEURIPS2023_6f1346ba,
 author = {Mahey, Guillaume and Chapel, Laetitia and Gasso, Gilles and Bonet, Cl\'{e}ment and Courty, Nicolas},
 booktitle = {Advances in Neural Information Processing Systems},
 editor = {A. Oh and T. Naumann and A. Globerson and K. Saenko and M. Hardt and S. Levine},
 pages = {35350--35385},
 publisher = {Curran Associates, Inc.},
 title = {Fast Optimal Transport through Sliced Generalized Wasserstein Geodesics},
 url = {https://proceedings.neurips.cc/paper_files/paper/2023/file/6f1346bac8b02f76a631400e2799b24b-Paper-Conference.pdf},
 volume = {36},
 year = {2023}
}

@inproceedings{rowland2019orthogonal,
  title={Orthogonal Estimation of Wasserstein Distances},
  author={Rowland, Mark and Hron, Jiri and Matthews, Alexander G. de G. and Ghahramani, Zoubin},
  booktitle={International Conference on Artificial Intelligence and Statistics},
  series={Proceedings of Machine Learning Research},
  volume={89},
  pages={186--195},
  year={2019},
  url={https://proceedings.mlr.press/v89/rowland19a.html}
}

@inproceedings{
chapel2026differentiable,
title={Differentiable Generalized Sliced Wasserstein Plans},
author={Laetitia Chapel and Romain Tavenard and Samuel Vaiter},
booktitle={The Thirty-ninth Annual Conference on Neural Information Processing Systems},
year={2026},
url={https://openreview.net/forum?id=hKKbtN4cp9}
}

\clearpage
\appendix

\section{Ablation Study}
\label{app:ablation}

We ablate two orthogonal design choices of \(\minGSGW\): whether the slicer family is
\textbf{linear} or \textbf{nonlinear}, and whether slicers are \textbf{independent} or
\textbf{dependent} (jointly optimized). This yields a $2\times2$ factorial ablation evaluated
on the realistic shape matching benchmark, reported in Table~\ref{tab:ablation}. Nonlinearity
and coupling each provide an independent gain in geodesic error: linear slicers fail to
capture curved geodesic geometry, while independent slicers collapse onto redundant directions
and lose diverse coverage of the metric-measure space. The full model (nonlinear + coupled)
benefits from both.

\begin{table}[ht!]
\centering
\small
\caption{Ablation of slicer design on the realistic shape matching benchmark.
Geo.: geodesic error $(\downarrow)$; Time: forward runtime in seconds $(\downarrow)$.
H, E, C denote Horse, Elephant, Cat.}
\setlength{\tabcolsep}{5pt}
\renewcommand{\arraystretch}{1.05}
\resizebox{\linewidth}{!}{%
\begin{tabular}{llrr|rr|rr|rr|rr|rr}
\toprule
& & \multicolumn{2}{c|}{H--H}
& \multicolumn{2}{c|}{E--E}
& \multicolumn{2}{c|}{C--C}
& \multicolumn{2}{c|}{H--E}
& \multicolumn{2}{c|}{C--H}
& \multicolumn{2}{c}{C--E} \\
\cmidrule(lr){3-4}\cmidrule(lr){5-6}\cmidrule(lr){7-8}
\cmidrule(lr){9-10}\cmidrule(lr){11-12}\cmidrule(lr){13-14}
Slicer & Relation
& Geo. & Time & Geo. & Time & Geo. & Time
& Geo. & Time & Geo. & Time & Geo. & Time \\
\midrule
\multirow{2}{*}{Linear}
& Independent
& 0.162 & 0.06 & 0.178 & 0.06 & 0.121 & 0.05
& 0.238 & 0.06 & 0.267 & 0.05 & 0.209 & 0.06 \\
& Dependent
& 0.134 & 0.06 & 0.147 & 0.06 & 0.098 & 0.05
& 0.203 & 0.06 & 0.231 & 0.05 & 0.178 & 0.06 \\
\midrule
\multirow{2}{*}{Nonlinear}
& Independent
& 0.108 & 0.08 & 0.119 & 0.08 & 0.081 & 0.07
& 0.172 & 0.08 & 0.198 & 0.07 & 0.153 & 0.08 \\
& Dependent
& \textbf{0.079} & \textbf{0.08} & \textbf{0.091} & \textbf{0.08} & \textbf{0.058} & \textbf{0.07}
& \textbf{0.138} & \textbf{0.08} & \textbf{0.162} & \textbf{0.07} & \textbf{0.124} & \textbf{0.08} \\
\bottomrule
\end{tabular}}
\label{tab:ablation}
\end{table}

\section{Toy correspondences.}
\label{app:toy_correspondences}
Figure~\ref{fig:toy_correspondences} compares exact GW against our method on four synthetic 2D-to-3D pairs. The experiment is qualitative: it shows that the nonlinear score construction is expressive enough to recover clean, globally coherent matches despite searching over a restricted coupling family. Unlike random one-dimensional projections, the shared map-lifting construction induces compatible orderings, yielding more structured correspondences across all four examples.

\begin{figure}[ht!]
    \centering
    \vspace{-0.1in}
    \includegraphics[width=1\textwidth]{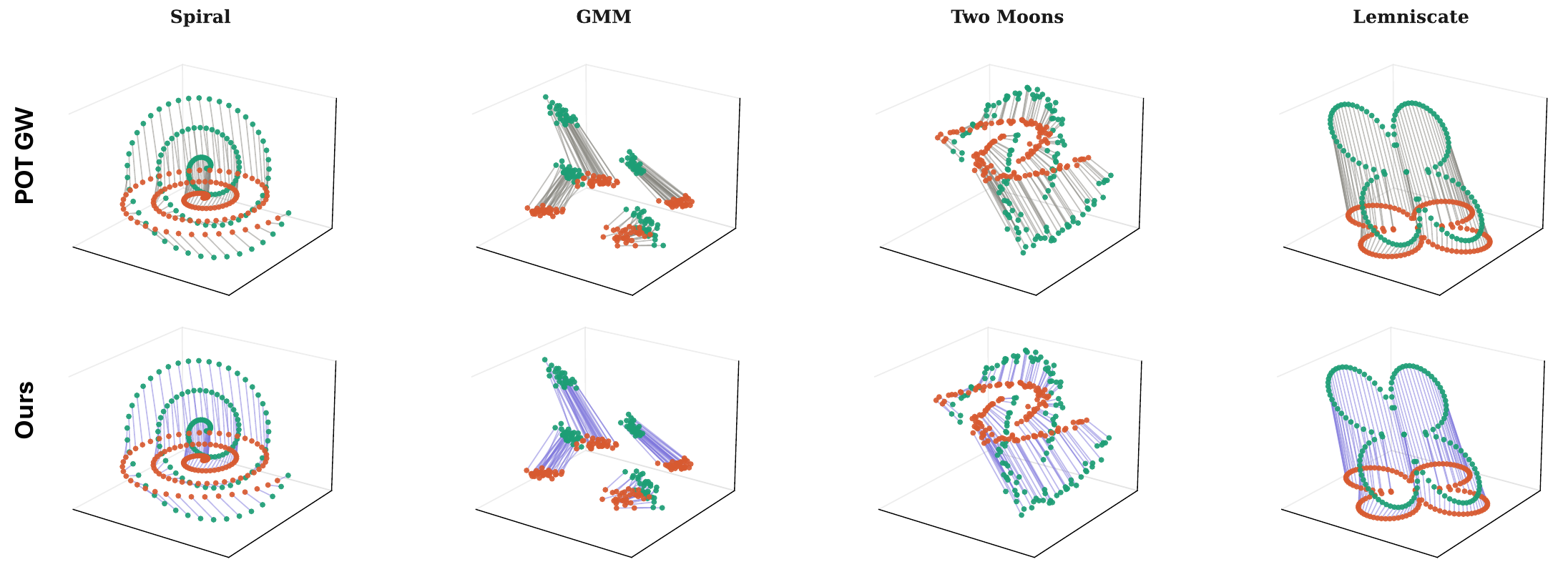}
    \caption{Qualitative comparison on four toy datasets. Top: GW correspondences. Bottom: ours. Source points lie on $z=0$; targets in $\mathbb{R}^3$. Our method produces cleaner, more structurally aligned matches.}
    \label{fig:toy_correspondences}
\end{figure}
\section{Experimental details and reproducibility}
\label{app:exp_details}

We report the implementation choices needed to reproduce the experiments. All methods use uniform marginals, and all optimization is implemented in PyTorch with CUDA used when available. Unless otherwise stated, each number reported in the main tables is averaged over three random seeds, \(\{42, 7, 77\}\).

\paragraph{Animal mesh matching.}
For the animal mesh correspondence experiments, we compute pairwise geodesic distance matrices from the mesh graph using Euclidean edge lengths, symmetrize the resulting distances, and compare our learned sorting based plan against other baselines. The effective hyperparameters are listed in Table~\ref{tab:animal_repro}.

\begin{figure}[t]
    \centering
    \includegraphics[width=\linewidth]{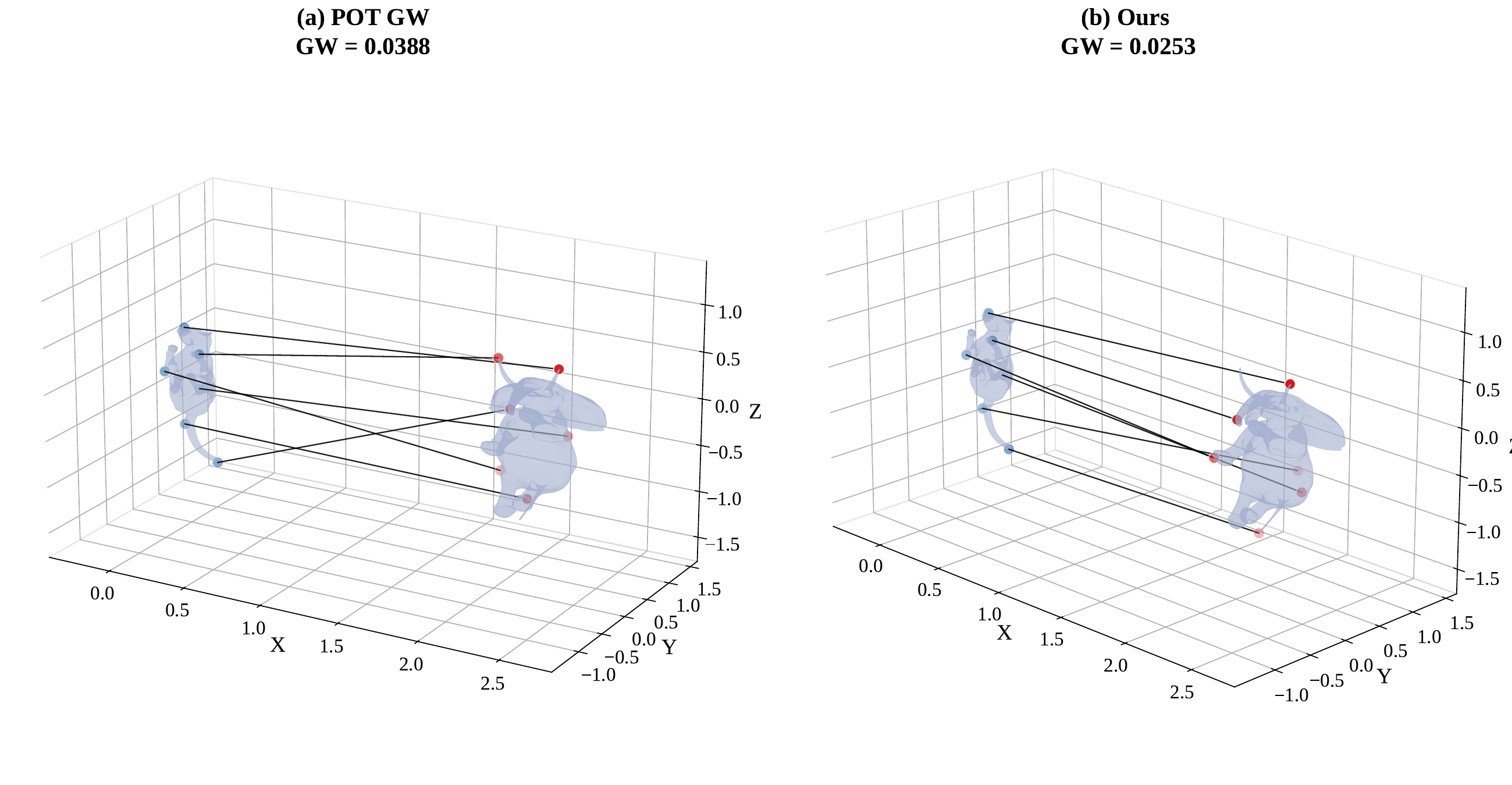}
    \vspace{-0.5in}
    \caption{\textbf{Shape correspondence on animal meshes.}
        Landmark correspondences (18 total, 6 shown) computed from geodesic distance matrices.
        Blue denotes source landmarks, red denotes predicted target landmarks, and black indicates correspondences.
            \textbf{(a)} POT GW. \textbf{(b)} Ours.\vspace{-0.1in}}
    \label{fig:shape_correspondence}
\end{figure}

\begin{table}[ht!]
\centering
\small
\setlength{\tabcolsep}{5pt}
\renewcommand{\arraystretch}{1.08}
\caption{Reproducibility details for animal mesh matching.}
\label{tab:animal_repro}
\begin{tabular}{lll}
\toprule
\textbf{Component} & \textbf{Setting} & \textbf{Value} \\
\midrule
Data & Mesh formats & \texttt{.obj}, \texttt{.off} \\
Landmarks & \texttt{n\_land} & 18 \\
Landmarks & \texttt{n\_rep} & 4 \\
Model & Hidden width & 1024 \\
Model & Depth & 6 \\
Model & Random Fourier features & 128 \\
Optimization & Optimizer & Adam \\
Optimization & Learning rate & $10^{-4}$ \\
Optimization & Steps & 1500 \\
Optimization & Gradient clipping & 1.0 \\
Sorting & Initial temperature & $10^{-4}$ \\
Baseline & POT solver & \texttt{gromov\_wasserstein} \\
Baseline & POT loss & \texttt{square\_loss} \\
Baseline & MSGW directions & 500 \\
\bottomrule
\end{tabular}
\end{table}

\paragraph{Horse mesh interpolation.}
For horse interpolation, meshes are centered and scaled by their maximum norm. Geodesic costs are computed on a symmetrized kNN graph with Euclidean edge weights and Dijkstra shortest paths. We train multiple random restarts and use the best hard plan for barycentric interpolation. The effective settings are given in Table~\ref{tab:horse_repro}.

\begin{table}[ht!]
\centering
\small
\setlength{\tabcolsep}{5pt}
\renewcommand{\arraystretch}{1.08}
\caption{Reproducibility details for horse mesh interpolation.}
\label{tab:horse_repro}
\begin{tabular}{lll}
\toprule
\textbf{Component} & \textbf{Setting} & \textbf{Value} \\
\midrule
Data & Mesh indices & $\{0,1,2\}$ \\
Data & Mesh formats & \texttt{.npy}, \texttt{.off} \\
Geodesics & Graph & symmetrized kNN \\
Geodesics & Neighbors & 20 \\
Geodesics & Solver & Dijkstra shortest path \\
Geodesics & Normalization & $[0,1]$ \\
Model & Hidden width & 512 \\
Model & Depth & 4 \\
Model & Expansion & 2 \\
Model & Dropout & 0.0 \\
Optimization & Optimizer & AdamW \\
Optimization & Steps & 1000 \\
Optimization & Restarts & 3 \\
Optimization & Learning rate & $3\times 10^{-3}$ \\
Optimization & Weight decay & $10^{-4}$ \\
Optimization & Warmup steps & 50 \\
Optimization & Gradient clipping & 5.0 \\
Annealing & \(\alpha_{\mathrm{start}}\) & 1.0 \\
Annealing & \(\alpha_{\mathrm{end}}\) & 0.03 \\
\bottomrule
\end{tabular}
\end{table}

\paragraph{Amortized ShapeNet matching.}
For the amortized setting, we train a shared model that predicts scalar scores for each input pair and induces a transport plan through differentiable sorting. The input point features include coordinates, normals, and an intrinsic descriptor obtained from sorted squared intra shape distances. The experiment uses the following ShapeNet categories: airplane, bag, cap, car, chair, earphone, guitar, knife, lamp, laptop, motorbike, mug, pistol, rocket, skateboard, and table. The training and architecture details needed for reproduction are summarized in Table~\ref{tab:shapenet_repro}.

\begin{table}[ht!]
\centering
\small
\setlength{\tabcolsep}{4pt}
\renewcommand{\arraystretch}{1.08}
\caption{Reproducibility details for amortized ShapeNet matching.}
\label{tab:shapenet_repro}
\begin{tabularx}{\linewidth}{p{0.24\linewidth}p{0.35\linewidth}X}
\toprule
\textbf{Component} & \textbf{Setting} & \textbf{Value} \\
\midrule
Data & Splits & \texttt{train} / \texttt{val} \\
Data & Points per shape & 1024 \\
Data & Batch size / workers & 4 / 4 \\
Data & Pairs per epoch / validation pairs & 2000 / 400 \\
Features & Input features & Intrinsic descriptor \(\phi\) \\
Features & Intrinsic descriptor dimension & 64 \\
Features & Coordinate normalization & centering and max norm scaling \\
Model & Token / latent dimension & 70 / 256 \\
Model & Attention heads & 4 \\
Model & Layer pattern & self, self plus cross, self plus cross \\
Model & Score head & symmetric cross-attending head \\
Optimization & Optimizer & AdamW \\
Optimization & Epochs & 500 \\
Optimization & Learning rate / weight decay & \(10^{-3}\) / \(10^{-5}\) \\
Optimization & Warmup epochs / gradient clipping & 5 / 1.0 \\
Annealing & \(\alpha_{\mathrm{start}}\) / \(\alpha_{\mathrm{end}}\) & 0.05 / 0.005 \\

\bottomrule
\end{tabularx}
\end{table}

\paragraph{Toy correspondences.}
The toy correspondence experiment is qualitative and is included to visualize the induced correspondence structure. It uses four planar source examples mapped to target point clouds in \(\mathbb{R}^3\), with exact GW used as a reference baseline. Since this experiment is illustrative, we do not include additional implementation specific details beyond those needed to interpret the figure.

\section{Amortized \texorpdfstring{\(\minGSGW\)}{minGSGW} Architecture}
\label{app:arch}

We describe the full architecture of the amortized matcher introduced in
Section~\ref{sec:shapenet_amortized}.  The network maps intrinsic token
collections \(\Phi^X\in\mathbb{R}^{N\times d_{\mathrm{in}}}\) and
\(\Phi^Y\in\mathbb{R}^{M\times d_{\mathrm{in}}}\) to score vectors
\(s^\theta\in\mathbb{R}^N\) and \(t^\theta\in\mathbb{R}^M\), which induce
the soft coupling via LapSum~\citep{lapsum2025}.  It consists of five
components: a shared token embedding, a two-stream transformer encoder,
a symmetric pair-context branch, a cross-attending score head, and the
LapSum differentiable assignment layer.

\paragraph{Notation.}
Let \(\mathrm{SA}(H)=\mathrm{MHA}(H,H,H)\) denote self-attention and
\(\mathrm{CA}(H,H')=\mathrm{MHA}(H,H',H')\) cross-attention, where
\(\mathrm{MHA}\) uses softmax multi-head attention.  We define four
transformer block types, each with a GELU MLP and post-LayerNorm
residuals:

\begin{itemize}[leftmargin=1.5em,topsep=2pt,itemsep=1pt]
  \item \(\mathrm{S}\): self-attention only, applied independently to each stream with \emph{tied} weights.
  \item \(\mathrm{C}\): cross-attention only; each stream attends to the other.
  \item \(\mathrm{SC}\): self-attention followed by cross attention.
  \item \(\mathrm{CS}\): cross-attention followed by self-attention.
\end{itemize}
All blocks within the encoder share weights across the two streams.

\paragraph{Token embedding.}
A shared linear map \(\mathrm{E}:\mathbb{R}^{d_{\mathrm{in}}}\to\mathbb{R}^d\)
embeds each token independently:
\[
H_X^{(0)} = \mathrm{E}(\Phi^X)\in\mathbb{R}^{N\times d},
\qquad
H_Y^{(0)} = \mathrm{E}(\Phi^Y)\in\mathbb{R}^{M\times d}.
\]

\paragraph{Two-stream encoder.}
The encoder follows the block pattern
\([\mathrm{S},\,\mathrm{SC},\,\mathrm{SC}]\),
producing contextualised representations \((H_X, H_Y)\).
In all ShapeNet experiments we use hidden dimension \(d=256\) and
\(h=4\) attention heads.

\paragraph{Symmetric pair-context branch.}
To inject a global summary of the pair that is symmetric under swapping
\(X\leftrightarrow Y\), a shared set encoder
\(\mathcal{C}\) (two-layer pointwise MLP, weighted mean pooling, linear
projection to \(\mathbb{R}^{d_c}\)) produces
\[
c
=
\tfrac{1}{2}\bigl(\mathcal{C}(\Phi^X)+\mathcal{C}(\Phi^Y)\bigr)
\in\mathbb{R}^{d_c}.
\]
This vector is broadcast and concatenated with the encoder outputs, then
projected back to width \(d\) via a shared map
\(W_c:\mathbb{R}^{d+d_c}\to\mathbb{R}^d\):
\[
\widetilde{H}_X = W_c\!\left[H_X \,\|\, \mathbf{1}_N c^\top\right],
\qquad
\widetilde{H}_Y = W_c\!\left[H_Y \,\|\, \mathbf{1}_M c^\top\right].
\]
Symmetry of \(c\) ensures that swapping the two inputs produces
identically conditioned streams, consistent with the symmetry constraint
\(G_\theta(Y,X)=G_\theta(X,Y)^\top\) in
\eqref{eq:maxmin_sym}.

\paragraph{Cross-attending score head.}
A depth-two \(\mathrm{SC}\) stack with \emph{separate} parameters from
the encoder further mixes the context-conditioned streams:
\[
(\bar{H}_X,\,\bar{H}_Y)
=
\mathrm{SC}_2\circ\mathrm{SC}_1(\widetilde{H}_X,\,\widetilde{H}_Y).
\]
A shared scalar readout \(w:\mathbb{R}^d\to\mathbb{R}\) then produces
the score vectors:
\[
s^\theta = w(\bar{H}_X)\in\mathbb{R}^N,
\qquad
t^\theta = w(\bar{H}_Y)\in\mathbb{R}^M.
\]
Using a \emph{shared} readout enforces the design principle of
\(\minGSGW\): the two score systems are coupled through a common
scalarization rather than chosen independently.

\paragraph{LapSum differentiable assignment.}
The scores \((s^\theta, t^\theta)\) are passed to
LapSum~\citep{lapsum2025}, which maps them to soft permutation matrices
\(P_{X,\tau}\in\mathbb{R}^{N\times N}\) and
\(P_{Y,\tau}\in\mathbb{R}^{M\times M}\) at temperature \(\tau\).
These induce the soft matching plan
\[
\pi_\tau^\theta
=
P_{X,\tau}\,\Pi_0\,P_{Y,\tau}^\top,
\]
where \(\Pi_0\) is the canonical sorted plan, as defined in
Section~\ref{sec:method}.
At inference time we set \(\tau\to 0\), recovering the hard plan
\(\pi^\theta\).

\paragraph{Intrinsic tokenisation details.}
The set encoder \(\rho\) maps an unordered distance profile
\(\mathcal{D}_i^X=\{\|x_i-x_k\|^2:k\neq i\}\) to a fixed-size vector.
Concretely, \(\rho\) sorts the \(K\) nearest squared distances and
passes the resulting \(K\)-dimensional vector through a two-layer MLP
with GELU activations, yielding
\(\phi_i^X\in\mathbb{R}^{d_{\mathrm{in}}}\).
We use \(K=32\) in all experiments.  Because sorting is applied to
distances within a single point cloud and the MLP weights are shared,
the tokens are rigid-motion invariant and permutation equivariant as
argued in Section~\ref{sec:shapenet_amortized}.

\end{document}